\documentclass[runningheads,a4paper]{llncs}
\usepackage{amssymb}
\usepackage{graphicx}

\usepackage{amssymb}
\usepackage{amsmath}
\usepackage{multirow}
\usepackage{algorithmic}
\usepackage{algorithm}
\usepackage{latexsym}

\usepackage{url}

\newcommand{\cF}{\mathcal{F}}
\newcommand{\cG}{\mathcal{G}}
\newcommand{\cH}{\mathcal{H}}

\newcommand{\cR}{\mathcal{R}}

\newcommand{\cT}{\mathcal{T}}
\newcommand{\cX}{\mathcal{X}}
\newcommand{\cY}{\mathcal{Y}}

\newcommand{\cC}{\mathcal{C}}

\newcommand{\cO}{\mathcal{O}}

\newcommand{\N}{{\rm I}\kern-0.18em{\rm N}}

\newcommand{\R}{{\rm I}\kern-0.18em{\rm R}}
\newcommand{\h}{{\rm I}\kern-0.18em{\rm H}}
\newcommand{\K}{{\rm I}\kern-0.18em{\rm K}}
\newcommand{\p}{{\rm I}\kern-0.18em{\rm P}}
\newcommand{\E}{{\rm I}\kern-0.18em{\rm E}}
\newcommand{\Z}{{\rm Z}\kern-0.18em{\rm Z}}

\newcommand{\1}{{\rm 1}\kern-0.25em{\rm I}}

\newcommand{\pn}{\p_{\kern-0.25em n}}
\newcommand{\pnm}{\p_{\kern-0.25em n,m}}
\newcommand{\psubm}{\p_{\kern-0.25em m}}

\newtheorem{MyDefinition}{Definition}
\newtheorem{MyLemma}{Lemma}
\newtheorem{MyTheorem}{Theorem}

\newtheorem{MyRemark}{Remark}

\begin{document}

\mainmatter  

\title{Nonparametric Unsupervised Classification}

\titlerunning{Nonparametric Unsupervised Classification}

\author{Yingzhen Yang \and Thomas S. Huang}
\institute{Department of Electrical and Computer Engineering\\
University of Illinois at Urbana-Champaign, USA\\
{\tt \{yyang58,huang\}@ifp.uiuc.edu}}

\toctitle{Nonparametric Unsupervised Classification}
\tocauthor{Yingzhen Yang \hspace{.1in} Thomas S. Huang}
\maketitle

\begin{abstract}
Unsupervised classification methods learn a discriminative classifier from unlabeled data, which has been proven to be an effective way of simultaneously clustering the data and training a classifier from the data. Various unsupervised classification methods obtain appealing results by the classifiers learned in an unsupervised manner. However, existing methods do not consider the misclassification error of the unsupervised classifiers except unsupervised SVM, so the performance of the unsupervised classifiers is not fully evaluated. In this work, we study the misclassification error of two popular classifiers, i.e. the nearest neighbor classifier (NN) and the plug-in classifier, in the setting of unsupervised classification. The upper bound for the misclassification error of both classifiers involves only pairwise similarity between the data points. We prove that the error of the plug-in classifier is asymptotically bounded by the weighted volume of cluster boundary \cite{NarayananBN06}. Also, the normalized graph Laplacian from the induced similarity kernel recovers different types of transition kernels for Diffusion maps \cite{Coifman05,Coifman06}, which reveals the close relationship between manifold learning and unsupervised classification. We show that with the normalized graph Laplacian, the similarity kernel induced by the misclassification error of the plug-in classifier corresponds to the Fokker-Planck operator; and the similarity kernel induced by the volume of misclassified region by the plug-in classifier correspond to the Laplace-Beltrami operator on the data manifold.
\end{abstract}

\section{Introduction}
Clustering methods partition the data into a set of self-similar clusters. Representative clustering methods include K-means \cite{HartiganW79} which minimizes the within-cluster dissimilarities, spectral clustering \cite{Ng01} which identifies clusters of more complex shapes lying on some low dimensional manifolds, and statistical modeling method \cite{Fraley02} approximates the data by a mixture of parametric distribution.

On the other hand, viewing clusters as classes, recent works on unsupervised classification learn a classifier from unlabeled data, and they have established the connection between clustering and multi-class classification from a supervised learning perspective. \cite{XuNLS04} learns a max-margin two-class classifier in an unsupervised manner. Such method is known as unsupervised SVM, whose theoretical property is further analyzed in \cite{Karnin12}. Also, \cite{AgakovB05} and \cite{GomesKP10} learn the kernelized Gaussian classifier and the kernel logistic regression classifier respectively. Both \cite{GomesKP10} and \cite{BridleHM91} adopt the entropy of the posterior distribution of the class label by the classifier to measure the quality of the learned classifier, and the parameters of such unsupervised classifiers can be computed by continuous optimization. More recent work presented in \cite{SugiyamaYKH11} learns an unsupervised classifier by maximizing the mutual information between cluster labels and the data, and the Squared-Loss Mutual Information is employed to produce a convex optimization problem. However, previous methods either do not consider the misclassification error of the learned unsupervised classifiers, one of the most important performance measures for classification, or only minimizes the error of unsupervised SVM \cite{XuNLS04}. Therefore, the performance of the unsupervised classifier is not fully evaluated. Although Bengio et al. \cite{BengioPVDRO03} analyze the out-of-sample error for unsupervised learning algorithms, their method focuses on lower-dimensional embedding of the data points and does not train a classifier from unlabeled data.

In contrast, we analyze the unsupervised nearest neighbor classifier (NN) and the plug-in classifier from unlabeled data by the training scheme for unsupervised classification introduced in \cite{XuNLS04}, and derive the bound for their misclassification error. Although the generalization properties of the NN and the plug-in classifier have been extensity studied \cite{Cover67,Audibert07}, to the best of our knowledge most analysis focuses on the case of average generalization error. Unsupervised classification methods, such as unsupervised SVM \cite{XuNLS04}, measure the quality of a specific data partition by its associated misclassification error, so we derive the data dependent misclassification error with respect to fixed training data. The resultant error bound comprises pairwise similarity between the data points, which also induces the similarity kernel over the data. We prove that the error of the plug-in classifier is asymptotically bounded by the (scaled) weighted volume of cluster boundary \cite{NarayananBN06}, and the latter is designed to encourage the cluster boundary to avoid high density regions following the Low Density Separation assumption \cite{Chapelle05}.

Building a graph where the nodes represent the data and the edge weight is set by the induced similarity kernel, clustering by minimizing the error bound of the two unsupervised classifiers reduces to (normalized) graph-cut problems, which can be solved by normalized graph Laplacian (or Normalized Cut \cite{Shi00}). The normalized graph Laplacian from the similarity kernel induced by the upper bound for the error (or the volume of the misclassified region) of the plug-in classifier renders a certain type of transition kernel for Diffusion maps \cite{Coifman05,Coifman06}, and the Fokker-Planck operator (or the Laplace-Beltrami operator) is recovered by the infinitesimal generator of the corresponding Markov chains. It is interesting to observe that the volume of the misclassified region is independent of the marginal data distribution, which is consistent with the fact that the Laplace-Beltrami operator only captures the geometric information. This implies close relationship between manifold learning and unsupervised classification.

The rest part of this paper is organized as follows. We first introduce the formulation of unsupervised classification in Section~\ref{sec::model}, then derive the error bound for the unsupervised NN and plug-in classifiers and explain the connection to other related methods in Section~\ref{sec::mainresults}. We conclude the paper in Section~\ref{sec::conclusion}.

\section{The Model}
\label{sec::model}
\setcounter{equation}{0}
We first introduce the notations in the formulation of unsupervised classification.
Let $(X,Y)$ be a random couple with joint distribution $P_{XY}$, where $X
\in \cX \subset \R^d$ is a
vector of $d$ features and $Y \in \left\{ {1,2,...,Q} \right\}$ is a label indicating the class to
which $X$ belongs. We assume that $\cX$ is bounded by ${\left[ { - {M_0},{M_0}} \right]^d}$. The sample
$(X_1, Y_1), \ldots, (X_n, Y_n)$ are independent
copies of $(X,Y)$, and we only observe $\left\{ {{X_l}} \right\}_{l = 1}^n$. Suppose $P_X$ is the induced marginal distribution of $X$.



\subsection*{The Training Scheme for Unsupervised Classification}

The training scheme introduced by the unsupervised SVM \cite{XuNLS04,Karnin12} forms the basis for learning a classifier from unlabeled data in a principled way. With any hypothetical labeling $\cY = \{Y_l\}_{l=1}^n$, we can build the corresponding training data ${S_C} \triangleq \left\{ {\left( {{C_i},i} \right)} \right\}_{i = 1}^Q$ for a potential classifier, where $C_i=\{X_l : Y_l=i,1 \le l \le n\}$ is the data with label $i$ and $\{C_i\}_{i=1}^Q$ is a partition of the data. In this way, the quality of a labeling $\{Y_l\}_{l=1}^n$, or equivalently a data partition, can be evaluated by the misclassification error of the classifier learned from the corresponding training data $S_C$\footnote{Two labelings are equivalent if they produce the same data partition $\{C_i\}_{i=1}^Q$, and we refer to $S_C$ as the data partition in the following text.}. Existing methods \cite{XuNLS04,Karnin12} perform clustering by searching for the data partition with minimum associated misclassification error. We adopt this training scheme for unsupervised classification, and analyze the misclassification error of the classifier learned from any fixed data partition (corresponding to a hypothetical labeling $\{Y_i\}_{i=1}^n$ ).

It is worthwhile to mention that previous unsupervised classification methods \cite{SugiyamaYKH11,GomesKP10} circumvent the above combinatorial unsupervised training scheme by learning a probabilistic classifier from the whole data, so they can not evaluate the classification performance of the learned classifier in the learning procedure. Rather than learning the unsupervised SVM \cite{XuNLS04,Karnin12}, we study the misclassification error of the unsupervised NN and plug-in classifiers, revealing the theoretical property of these popular classifiers in the setting of unsupervised classification.

\subsection*{The Misclassification Error}
By the training scheme for unsupervised classification, the misclassification error (or the generalization error) of the classifier $F_{S_C}$ learned from the training data $S_C$ is:
\begin{equation}\label{eq:generalizationerror}
R\left(F_{S_C}\right) \triangleq P\left( F_{S_C} \neq Y \right)
\end{equation}

\noindent In order to evaluate the quality of the data partition $S_C$, we should estimate $R(F_{S_C})$ with any fixed $S_C$ rather than the average generalization error ${\E_{S_C}}\left[{R\left(F_{S_C}\right)}\right]$. $F_{S_C}$ indicates either NN or plug-in classifier from $S_C$, and ${{F_{S_C}}\left( X \right)}$ is the classification function which returns the class label of a sample $X$. We also let $f$ be the probabilistic density function of $P_X$, ${\eta^{(i)}}\left( x \right)$ be the regression function of $Y$ on $X=x$, i.e. ${\eta^{(i)}\left( x \right)} = P\left[ {Y = i\left| {X = x} \right.} \right]$, and $\left( {{f^{(i)}},{\pi^{(i)}}} \right)$ be the class-conditional density function and the prior for class $i$ ($f=\sum\limits_{i}{{\pi^{(i)}}{f^{(i)}}}$). Let $f,\{f^{(i)}\}_{i=1}^Q, \{\eta^{(i)}\}_{i=1}^Q$ be measurable functions, and there are further assumptions on $f$:

\textbf{(A1)} $f$ is bounded, i.e. $0 < f_{min} \le f \le f_{max}$.


\textbf{(A2)} $f$ is H\"{o}lder-$\gamma$ smooth: $\left| {f\left( x \right) - f\left( y \right)} \right| \le c{\left\| {x - y} \right\|}^{\gamma}$ where $c$ is the H\"{o}lder constant and $\gamma > 0$.

Because we will estimate the underlying probabilistic density function frequently in the following text, we introduce the non-parametric kernel density estimator of $f$ as below:
\begin{align}\label{eq:kde}
&{{\hat f}_{n,h}}\left( x \right) = \frac{1}{{n}}\sum\limits_{l = 1}^n {K_{h}\left( {x - {X_l}} \right)}
\end{align}
where
\begin{equation}\label{eq:kernel}
K_{h}\left( x \right) = K\left(\frac{x}{h}\right), K\left(x\right) \triangleq \frac{1}{{{{\left( {2\pi } \right)}^{{d \mathord{\left/
 {\vphantom {d 2}} \right.
 \kern-\nulldelimiterspace} 2}}}}}{e^{ - \frac{{{{\left\| x \right\|}^2}}}{2}}}
\end{equation}
\noindent and $K_{h}\left(\cdot\right)$ is the isotropic Gaussian kernel with bandwidth $h$. We introuduce one assumption on the kernel bandwidth sequence $\{h_n\}_{n=1}^{\infty}$:

\textbf{(B)} (1) $h_n \searrow 0$, 
(2) $\frac{-\log{h_n}}{nh_n^{d+2\gamma}} \to 0$, (4) $\frac{{ - \log {h_n}}}{{\log \log n}} \to \infty$, (5)$h_n^d < ah_{2n}^d$ for some $a>0$.

\cite{Gine02} proves that the kernel density estimator (\ref{eq:kde}) almost sure uniformly converges to the underlying density:
\begin{MyTheorem}\label{theorem::kdeconvergence}
(Theorem 2.3 in Gine et al. \cite{Gine02}, in a slighted change form) Under the assumption (A1)-(A2), suppose the kernel bandwidth sequence $\{h_n\}_{n=1}^{\infty}$ satisfies assumption (B), then with probability $1$
\begin{align}\label{eq:kdeconvergence}
&\mathop {\varlimsup }\limits_{n \to \infty } {h_n^{-\gamma}} \left\| {{{\hat f}_{n,h_n}}\left( x \right) - f\left( x \right)} \right\|_{\infty} = {\cC}_{K}
\end{align}
\noindent
where $\cC_{K} = \int_{\cX}{\left\|x\right\|^{\gamma}K\left(x\right) dx}$
\end{MyTheorem}

Similarly, the kernel density estimator of the class-conditional density function $f^{(i)}$ is
\begin{align}\label{eq::fjkde}
&{{\hat f}_{n,h}^{(i)}\left( x \right)} = \frac{1}{{n{\pi^{(i)}}}}\sum\limits_{l = 1}^n {{K_{h}}\left( {x - {X_l}} \right)\1_{\{Y_l=i\}}}
\end{align}

\noindent ${\1}$ is an indicator function. By the similar argument for the kernel density estimator (\ref{eq:kde}), under the assumption (B), (C1)-(C2), we have the almost sure uniform convergence for $\hat f_n^{(i)}$, i.e. $\mathop {\varlimsup }\limits_{n \to \infty } {h_n^{-\gamma}} \left\| {{{\hat f}_{n,h_n}^{(i)}}\left( x \right) - f^{(i)}\left( x \right)} \right\|_{\infty} = {\cC}_{K}$ with probability $1$ for $1 \le i \le Q$.

\textbf{(C1)} $\{f^{(i)}\}_{i=1}^Q$ are bounded, i.e. $0 < f_{min}^{(i)} \le f^{(i)} \le f_{max}^{(i)}, 1 \le i \le Q$.


\textbf{(C2)} $f^{(i)}$ is H\"{o}lder-$\gamma$ smooth: $\left| {f^{(i)}\left( x \right) - f^{(i)}\left( y \right)} \right| \le {c_i}{\left\| {x - y} \right\|}^{\gamma}$ where $c_i$ is the H\"{o}lder constant, $1 \le i \le Q$, and $\gamma > 0$.

Note that the assumption (C2) indicates (A2).

\section{Main Results}\label{sec::mainresults}

We prove the error bound for the unsupervised NN and plug-in classifiers, and then show the connection to other related methods in this section.

\subsection{Unsupervised Classification By Nearest Neighbor}

Since the NN rule makes hard decision for a given datum, we introduce the following soft NN cost function which converges to the NN classification function, similar to the one adopted by Neighbourhood Components Analysis \cite{GoldbergerRHS04}:
\begin{MyDefinition}
The soft NN cost function is defined as
\begin{equation}\label{eq:softnncost}
{\hat {NN}_{S_C,{h^*}}}\left( {x,i} \right) = \frac{{\sum\limits_{l = 1}^N {{K_{{h^*}}}\left( {x - {X_l}} \right){{\1}_{\{Y_l=i\}}}} }}{{\sum\limits_{l = 1}^N {{K_{{h^*}}}\left( {x - {X_l}} \right)} }}
\end{equation}
where ${\hat {NN}_{S_C,{h^*}}}\left( {x,i} \right)$ represents the probability that the datum $x$ is assigned to class $i$ by the soft NN rule $\hat {NN}_{S_C,{h^*}}$ learned from $S$.
\end{MyDefinition}

Then we have the misclassification error of the soft NN:
\begin{MyLemma}\label{lemma::generalizationerrorlemma}
The misclassification error of the soft NN is given by
\begin{equation}\label{eq:softnngeneralizationerror}
R(\hat {NN}_{{S_C},{h^*}}) = \sum\limits_{i,j = 1,...,Q,i \ne j} {{{\E}_X}\left[ {{\eta^{(i)}}\left( X \right){{\hat {NN}}_{S_C,{h^*}}}\left( {X,j} \right)} \right]}
\end{equation}
\end{MyLemma}

Lemma~\ref{lemma::generalizationerrorlemma} can be proved by the definition of misclassification error. Lemma~\ref{lemma::generalizationerrorboundlemma} shows that, with a large probability, the error of the soft NN (\ref{eq:softnngeneralizationerror}) is bounded. To facilitate our analysis, we introduce the cover of ${\mathcal X}$ as below:

\begin{MyDefinition}\label{definition::taucover}
The $\tau$-cover of the set ${\mathcal X}$ is a sequence of sets $\left\{ {{P_1},{P_2},...,{P_L}} \right\}$ such that ${\mathcal X} \subseteq \bigcup\limits_{r = 1}^L {{P_r}}$ and each ${P_r}$ is a box of length $\tau$ in $\R^d$, $1 \le r \le L$.
\end{MyDefinition}

\begin{MyLemma}\label{lemma::generalizationerrorboundlemma}
Under the assumption (A1), (C1)-(C2), suppose the kernel bandwidth sequence $\{h_n\}_{n=1}^{\infty}$ satisfies assumption (B), then with probability greater than $1 - 2L{e^{-M_{h^*}}}$ the misclassification error of the soft NN, i.e. $R(\hat {NN}_{{S_C},{h^*}})$, satisfies:
\begin{align}\label{eq:1nngeneralizationerrorbound}
&\frac{1}{n}\sum\limits_{i \ne j} {\sum\limits_{l = 1}^n {{{\1}_{\{Y_l = j \}}}\int_{\mathcal X} {\frac{{{\pi^{(i)}}{\hat f_{n,h_n}^{(i)}}\left( x \right){K_{{h^*}}}\left( {x - {X_l}} \right)}}{{{{\hat \E}_Z}\left[ {{K_{{h^*}}}\left( {x - Z} \right)} \right] + \cO\left(h_n^{\gamma}\right) + \widetilde \varepsilon }}} } } dx + \cO\left(h_n^{\gamma}\right) \le R(\hat {NN}_{{S_C},{h^*}}) \nonumber \\
&\le \frac{1}{n}\sum\limits_{i \ne j} {\sum\limits_{l = 1}^n{{{\1}_{\{Y_l =j\}}}\int_{\mathcal X} {\frac{{{\pi^{(i)}}{\hat f_{n,h_n}^{(i)}}\left( x \right){K_{{h^*}}}\left( {x - {X_l}} \right)}}{{{{\hat \E}_Z}\left[ {{K_{{h^*}}}\left( {x - Z} \right)} \right] + \cO\left(h_n^{\gamma}\right) - \widetilde \varepsilon }}} } } dx + \cO\left(h_n^{\gamma}\right)
\end{align}
\noindent where ${{\hat \E}}\left[Z\right] = \int_{\cX}{z\hat f_{n,h_n}\left(z\right)} dx$, $L$ is the size of the $\tau$-cover of ${\mathcal X}$, $M_{h^*}=2n{\left(2\pi\right)}^{d}{h^*}^{2d}{\varepsilon ^2}$, $\widetilde \varepsilon  = {T_1{\left( h^* \right)}}{\sqrt d }\tau  + c\left({\sqrt d}\tau\right)^{\gamma}+\varepsilon$, $T_1{\left( {h^*} \right)} = \frac{1}{{e{}^{{1 \mathord{\left/
 {\vphantom {1 2}} \right. \kern-\nulldelimiterspace} 2}}{{\left( {2\pi } \right)}^{{d \mathord{\left/ {\vphantom {d 2}} \right.
 \kern-\nulldelimiterspace} 2}}}{h^*}^{d + 1}}}$, $h^*,\tau,\varepsilon$ are small enough such that $\widetilde \varepsilon  < {f_{\min }}$.
\end{MyLemma}
\begin{proof}
Let $\left\{ {{P_1},{P_2},...,{P_L}} \right\}$ be the $\tau$-cover of the set ${\mathcal X}$. Suppose $L$ points $\left\{ {{{\widehat x}_r}} \right\}_{r = 1}^L$ are chosen from $\mathcal X$ and ${\widehat x_r} \in {P_r}$. For each $1 \le r \le L$, according to the Hoeffding's inequality
\begin{align}\label{eq:kernelsumprob}
\Pr \left[ {\left| {{T\left( {{{\widehat x}_r}} \right)} - {{\E}_Z}\left[ {{K_{{h^*}}}\left( {{{\widehat x}_r} - Z} \right)} \right]} \right| > \varepsilon } \right] < 2{e^{-M_{h^*}}}
\end{align}
where $T\left( x \right) \triangleq \frac{\sum\limits_{l = 1}^n {{K_{{h^*}}}\left( {x - {X_l}} \right)}}{n}$.
By the union bound, the probability that the above event happens for $\left\{ {{{\widehat x}_r}} \right\}_{r = 1}^L$ is less than $2L{e^{-M_{h^*}}}$. It follows that with probability greater than $1-2L{e^{-M_{h^*}}}$,
\begin{align}\label{eq:kernelsumprobunion}
\left| {{T\left( {{{\widehat x}_r}} \right)} - {{\E}_Z}\left[ {{K_{{h^*}}}\left( {{{\widehat x}_r} - Z} \right)} \right]} \right| \le \varepsilon
\end{align}
holds for any $1 \le r \le L$. For any $x \in \mathcal X$, $x \in {P_r}$ for some $P_r$, so that
\begin{align}\label{eq:Tdiff}
&\left| {T\left( x \right) - T\left( {{{\hat x}_r}} \right)} \right| \le T_1{\left( h^* \right)}\left\| {x - {{\hat x}_r}} \right\| \le
T_1{\left( h^* \right)}{\sqrt d }\tau
\end{align}
\noindent where $T_1{\left( {h^*} \right)} = \frac{1}{{e{}^{{1 \mathord{\left/
 {\vphantom {1 2}} \right. \kern-\nulldelimiterspace} 2}}{{\left( {2\pi } \right)}^{{d \mathord{\left/ {\vphantom {d 2}} \right.
 \kern-\nulldelimiterspace} 2}}}{h^*}^{d + 1}}}$. Moreover,
\begin{align}\label{eq:Expectationdiff}
&\left|{{\E}_Z}\left[ {{K_{{h^*}}}\left( {{{\widehat x}_r} - Z} \right)} \right] - {{\E}_Z}\left[ {{K_{{h^*}}}\left( x - Z \right)} \right] \right| \le c\left({\sqrt d}\tau\right)^{\gamma}
\end{align}
Combining (\ref{eq:kernelsumprobunion}), (\ref{eq:Tdiff}) and (\ref{eq:Expectationdiff}),
\begin{align}\label{eq:Tcombine}
\left| {{T\left( x \right)} - {{\E}_Z}\left[ {{K_{{h^*}}}\left( x - Z \right)} \right]} \right| \le {T_1{\left( h^* \right)}}{\sqrt d }\tau  + c\left({\sqrt d}\tau\right)^{\gamma} + \varepsilon
\end{align}
By Theorem~\ref{theorem::kdeconvergence}, $\left\| {{{\hat f}_{n,h_n}}\left( x \right) - f\left( x \right)} \right\|_{\infty} = {\cO}\left({h_n^{\gamma}}\right)$. Similarly, $\left\| {{{\hat f}_{n,h_n}^{(i)}}\left( x \right) - f^{(i)}\left( x \right)} \right\|_{\infty} = {\cO}\left({h_n^{\gamma}}\right)$. Substituting ${{\hat f}_{n,h_n}}$ and ${{\hat f}_{n,h_n}^{(i)}}$ for $f$ and $f^{(i)}$, and applying (\ref{eq:Tcombine}), (\ref{eq:1nngeneralizationerrorbound}) is verified. 
\qed
\end{proof}


Denote the classification function of the NN by ${NN}_{S_C}$, it can be verified that $\mathop {\lim }\limits_{{h^*} \to 0}  {\hat {NN}_{S_C,{h^*}}}\left( {X,i} \right)={NN}_{S_C}\left(X\right)$. In order to approach the misclassification error of the NN, we construcut a sequence $\{h_n^*\} \to 0$. Letting $n \to \infty$, we further have the asymptotic misclassification errof of the soft NN:

\begin{MyTheorem}\label{theorem::generalizationerrorboundasymptotictheorem}
Let $\left\{ {h_n^*} \right\}_{n = 1}^\infty$ be a sequence such that $\mathop {\lim }\limits_{n \to \infty } h_n^* = 0$ and $h_n^* \ge {n^{ - d_0}}$ with $d_0 < \frac{1}{2d}$. Under the assumption (A1), (C1)-(C2), when $n \to \infty$, then with probability $1$,
\begin{align}\label{eq:1nngeneralizationerrorlimitbound}
\mathop {\lim }\limits_{n \to \infty } \{{R\left({\hat {NN}_{S_C,h_n^*}}\right)}-
\frac{1}{n^2}\sum\limits_{l < m} {{\theta _{lm}}{H_{lm}}}\}=0
\end{align}
\begin{footnotesize}
\begin{align}\label{eq:Hlm}
&{H_{lm}} = {K_{h_n}}\left( {{x_l} - {x_m}} \right)\left( {\frac{1}{{{{\hat f}_{n,h_n}}\left( {{X_l}} \right)}} + \frac{1}{{{{\hat f}_{n,h_n}}\left( {{X_m}} \right)}}} \right)
\end{align}
\noindent where $\hat f_{n,h_n}$ is the kernel density estimator defined by (\ref{eq:kde}) with kernel bandwidth sequence $\{h_n\}$ under the assumption (B), ${{\theta _{lm}}} = {\1}_{\{Y_l \ne Y_m\}}$ is a class indicator function such that ${\theta _{lm}} = 1$ if $X_l$ and $X_m$ belongs to different classes, and $0$ otherwise.
\end{footnotesize}
\end{MyTheorem}
\begin{proof}
In Lemma~\ref{lemma::generalizationerrorboundlemma}, let $\tau={\tau _n} = {\tau _0}{n^{ - {d_0}\left( {d + 1} \right)}}$, $\varepsilon=\varepsilon_n = {n^{-\varepsilon_0}}$ for ${\tau _0} > 0$ and $\varepsilon_0 < 1-2dd_0$.
Since $h_n^* \ge {n^{ - d_0}}$, $\widetilde \varepsilon  \le {\lambda _0}{\tau _0} + c{d^{{\gamma  \mathord{\left/
 {\vphantom {\gamma  2}} \right.
 \kern-\nulldelimiterspace} 2}}} {\tau_n^{\gamma}}+ \varepsilon_n$ with ${\lambda _0} = \frac{{\sqrt d }}{{{e^{{1 \mathord{\left/
 {\vphantom {1 2}} \right.
 \kern-\nulldelimiterspace} 2}}}{{\left( {2\pi } \right)}^{{d \mathord{\left/
 {\vphantom {d 2}} \right.
 \kern-\nulldelimiterspace} 2}}}}}$. $\tau _0$ is small enough such that ${\lambda _0}{\tau _0} + c{d^{{\gamma  \mathord{\left/
 {\vphantom {\gamma  2}} \right.
 \kern-\nulldelimiterspace} 2}}} {\tau_n^{\gamma}} + \varepsilon_n< {f_{\min }}$ for sufficiently large $N$.

Let $n \to \infty $ in the RHS of (\ref{eq:1nngeneralizationerrorbound}), note that $h_n \to 0$ and $L = \left(\frac{2M_0}{\tau_n}\right)^d$, then with probability $1$,
\begin{align}\label{eq::thm1seg2}
\varlimsup\limits_{n \to \infty }\{{R\left({\hat {NN}_{S_C,h_n^*}}\right)} - \frac{1}{n}\sum\limits_{i \ne j} {\sum\limits_{l = 1}^n{{{\1}_{\{Y_l =j\}}}{\frac{{{\pi^{(i)}}{\hat f_{n,h_n}^{(i)}}\left( X_l \right){K_{{h^*}}}\left( {x - {X_l}} \right)}}{\hat f_{n,h_n}\left(X_l\right)-{{\lambda}_0}{{\tau}_0}}} } } dx \} \le 0
\end{align}
Substitute $\hat f_{n,h_n}^{(i)}$ into (\ref{eq::thm1seg2}),
\begin{align}\label{eq::thm1seg3}
\varlimsup\limits_{n \to \infty }\{{R\left({\hat {NN}_{S_C,h_n^*}}\right)} - \frac{1}{n^2}\sum\limits_{l < m} {{\theta _{lm}}\left({\frac{{{K_{h_n}}\left( {{X_l} - {X_m}} \right)}}{{\hat f}_{n,h_n}\left( X_l \right)-{{\lambda}_0}{{\tau}_0}} + \frac{{{K_{h_n}}\left( {{X_l} - {X_m}} \right)}}{{\hat f}_{n,h_n}\left( X_m \right)-{{\lambda}_0}{{\tau}_0}}}\right)}\} \le 0
\end{align}
Since (\ref{eq::thm1seg3}) holds for arbitrarily small $\tau _0 > 0$, $\varlimsup\limits_{n \to \infty }\{{{{\hat R}_{{NN}_{S_C},{h_n^*}}}} -
\frac{1}{n^2}\sum\limits_{l < m} {{\theta _{lm}}{H_{lm}}}\} \le 0$ with probability $1$. Similarly $\varliminf\limits_{n \to \infty }\{{{{\hat R}_{{NN}_{S_C},{h_n^*}}}} -
\frac{1}{n^2}\sum\limits_{l < m} {{\theta _{lm}}{H_{lm}}}\} \ge 0$ with probability $1$, and (\ref{eq:1nngeneralizationerrorlimitbound}) is verified.
\qed
\end{proof}

By Theorem~\ref{theorem::generalizationerrorboundasymptotictheorem}, the misclassification error $\frac{1}{n^2}\sum\limits_{l < m} {{\theta _{lm}}{H_{lm}}}$ involves only pairwise terms, and $H_{lm}$ can be interpreted as the similarity kernel over $x_l$ and $x_m$ induced by the misclassification error of the unsupervised NN.

\subsection{Unsupervised Classification By Plug-in Classifier}

Next we will derive the misclassification error of the unsupervised plug-in classifier, that is, the classifier with the form
\begin{align}\label{eq::pluginrule}
F_n^{PI} \left( X \right) = \mathop {\arg \max }\limits_{1 \le i \le Q} {{\hat \eta}^{(i)}}\left( X \right)
\end{align}
\noindent where $\hat \eta_n^{(i)}$ is a nonparametric estimator of the regression function $\eta^{(i)}$, and we choose

\begin{align}\label{eq::etaestimator}
&\hat \eta_n^{(i)}\left(x\right) = \frac{\sum\limits_{l = 1}^n {K_{h_n}\left( {x - {X_l}} \right)}{\1}_{\{Y_l=i\}}}{n{\hat f}_n\left(x\right)}
\end{align}

Let $F^*$ be the Bayesian classifier which is a minimizer of the misclassification error of all classifiers, and $F^*\left(X\right)=\mathop {\arg \max }\limits_{1 \le i \le Q} {{\eta}^{(i)}}\left( X \right)$. Due to the almost sure uniform convergence of the kernel density estimator $\hat f_{n,h_n}^{(i)}$ and $\hat f_{n,h_n}$ under the assumption (A1), (B), (C1)-(C2), $\hat \eta_n^{(i)}$ converges almost sure uniformly to $\eta^{(i)}$, and $F_n^{PI}$ converges to the Bayesian classifier $F^*$ : $\mathop {\lim }\limits_{n \to \infty } {F_n^{PI}}=F^*$. It follows that $\mathop {\lim }\limits_{n \to \infty } {R\left(F_n^{PI}\right)}=R\left(F^*\right)$ by the dominant convergence theorem. It is also known that the excess risk of $F_n^{PI}$, namely $\E R\left(F_n^{PI}\right)-R\left(F^*\right)$, converges to $0$ of the order $n^{\frac{-\beta}{2\beta+d}}$ under some complexity assumption on the class of the regression functions $\Sigma$ with smooth parameter $\beta$ that $\{\eta^{(i)}\}$ belongs to \cite{Yang99,Audibert07}. Again, this result deals with the average generalization error and cannot be applied to deriving the data dependent misclassification error of unsupervised classification with fixed training data in our setting.

Similar to Lemma~\ref{lemma::generalizationerrorlemma}, it can be verified that
\begin{align}\label{eq:pluginerror}
&{R\left(F_n^{PI}\right)} = \sum\limits_{i,j = 1,...,Q,i \ne j} {{{\E}_X}\left[ {{\eta^{(i)}}\left( X \right){P\left[ {F_n^{PI}\left( X \right) = j} \right]}} \right]}
\end{align}
We then give the upper bound for the misclassification error of $F_n^{PI}$ in Lemma~\ref{lemma::pluginerrorboundlemma}.
\begin{MyLemma}\label{lemma::pluginerrorboundlemma}
Under the assumption (A1), (B), (C1)-(C2), the asymptotic misclassification error of the plug-in classifier $F_n^{PI}$ satisfies
\begin{align}\label{eq:pluginerrorbound}
&{R\left(F_n^{PI}\right)} \le  {R_n^{PI}} + \cO\left(h_n^{\gamma}\right) \\ &{R_n^{PI}} \triangleq 2\sum\limits_{i,j = 1,...,Q,i \ne j} {{{\E}_X}\left[ {{\hat \eta_n^{(i)}}\left( X \right){\hat \eta_n^{(j)}}\left( X \right)} \right]}
\end{align}
\noindent where $\left\{ {{\eta^{(i)}}} \right\}_{i = 1}^Q$ is the regression functions, and this bound is tight.
\end{MyLemma}
\begin{proof}
By Theorem~\ref{theorem::kdeconvergence}, for $1 \le i \le Q$
\begin{align}\label{eq:etaerror}
&\left\|\hat \eta_n^{(i)}-\eta^{(i)}\right\|_{\infty} = \cO\left(h_n^{\gamma}\right)
\end{align}
According to (\ref{eq:pluginerror}), ${R\left(F_n^{PI}\right)} = \sum\limits_{i,j = 1,...,Q,i \ne j} {{{\E}_X}\left[ {{\hat \eta_n^{(i)}}\left( X \right){P\left[ {F_n^{PI}\left( X \right) = j} \right]}} \right]} + \cO\left(h_n^{\gamma}\right)$.

Suppose the decision regions of $F_n^{PI}$ is $\{\cR_1,\cR_2,\ldots \cR_Q\}$, then on each $\cR_i$, $\hat \eta_n^{(i)} \ge \hat \eta_n^{(i^{'})}$ for any $i^{'} \ne i$, and
\begin{align}\label{eq:bayeserrorboundlemmaseg1}
&\sum\limits_{i,j = 1,...,Q,i \ne j} {{{\E}_X}\left[ {{\hat \eta_n^{(i)}}\left( X \right){P\left[ {F_n^{PI}\left( X \right) = j} \right]}} \right]} \\ \nonumber
&= \sum\limits_{i,j = 1,...,Q,i \ne j} {{{\E}_{X \in \cR_j}}\left[ {{\hat \eta_n^{(i)}}\left( X \right)}\cdot{\sum\limits_{k=1}^Q{\hat \eta_n^{(k)}\left(X\right)}} \right]}\\ \nonumber
&={{\E}_{X}\left[{\left(\sum\limits_{k=1}^Q{\hat \eta_n^{(k)}\left(X\right)}\right)^2}\right]}
-\sum\limits_{i=1}^Q{{\E}_{X \in \cR_i}\left[{{\hat \eta_n^{(i)}\left(X\right)\cdot}\sum\limits_{k=1}^Q{\hat \eta_n^{(k)}\left(X\right)}}\right]} \\ \nonumber
&\le {{\E}_{X}\left[{\left(\sum\limits_{k=1}^Q{\hat \eta_n^{(k)}\left(X\right)}\right)^2}\right]}
-\sum\limits_{i=1}^Q{{\E}_{X}\left[{{\left(\hat \eta_n^{(i)}\left(X\right)\right)^2}}\right]} \\\ \nonumber
&=2\sum\limits_{i,j = 1,...,Q,i \ne j} {{{\E}_X}\left[ {{\hat \eta_n^{(i)}}\left( X \right){\hat \eta_n^{(j)}}\left( X \right)} \right]}
\end{align}
So that (\ref{eq:pluginerrorbound}) is verified. Since the equality in (\ref{eq:bayeserrorboundlemmaseg1}) holds when ${\hat \eta_n^{(i)}}\equiv \frac{1}{Q}$ for $1 \le i \le Q$, the upper bound in (\ref{eq:pluginerrorbound}) is tight.
\qed
\end{proof}

Based on Lemma~\ref{lemma::pluginerrorboundlemma}, we can bound the error of the plug-in classifier from above by $R_n^{PI}$. In order to estimate the error bound $R_n^{PI}$, we introduce the following generalized kernel density estimator:

\begin{MyLemma}\label{lemma::gkdelemma}
Suppose $f$ is a probabilistic density function on $\cX \subset {\left[ { - {M_0},{M_0}} \right]^d}$ that satisfies assumption (A1)-(A2). $\left\{ {{X_l}} \right\}_{l = 1}^N$ are drawn i.i.d. according to $f$. Let $g$ be a  H\"{o}lder-$\gamma$ smooth continuous function defined on $\cal X$ with H\"{o}lder constant $g_0$, and $g$ is bounded, i.e. $0 < g_{min}\le g \le g_{max} $. Let $e = \frac{f}{g}$. Define the generalized kernel density estimator of $e$ as
\begin{align}\label{eq:gkde}
{\hat e_{n,h}} \triangleq \frac{1}{n}\sum\limits_{l = 1}^n {\frac{{K_{h}\left( {{x - {X_l}}} \right)}}{{g\left( {{X_l}} \right)}}}
\end{align}
When the kernel bandwidth sequence $\{h_n\}_{n=1}^{\infty}$ satisfies assumption (B), then the estimator ${\hat e_{n,h_n}}$ converges to $e$ almost sure uniformly, i.e. with probability $1$,
\begin{align}\label{eq:gkdeconvergence}
&\mathop {\lim }\limits_{n \to \infty } {h_n^{-\gamma}}\left\| {{{\hat e}_{n,h_n}}\left( x \right) - e\left( x \right)} \right\|_{\infty} = \cC_{K,f,g}
\end{align}
\noindent where $\cC_{K,f,g}=\frac{f_{max}+g_{max}}{g_{min}^2}\int_{\cX}{\left\|x\right\|^{\gamma}K\left(x\right)} dx$.
\end{MyLemma}
\begin{proof}
Define the class of functions on the measurable space $\left( \rm H, \cH\right)$: $$\cF \triangleq \{K\left(\frac{t-\cdot}{h}\right),t \in \R^d, h \ne 0\} \quad \cF_g \triangleq \{\frac{K\left(\frac{t-\cdot}{h}\right)}{g\left(\cdot\right)},t \in \R^d, h \ne 0\}$$
It is shown in \cite{Gine02,Dudley99} that $\cF$ is a bounded VC class of measurable functions with respect to the envelope $F$ such that $\left|u\right| \le F$ for any $u \in \cF$. Therefore, there exist positive numbers $A$ and $v$ such that for every
probability measure $P$ on $\left( H, \cH\right)$ for which $\int_{}{F^2} d{P} < \infty$ and any $0 < \tau < 1$,
\begin{align}\label{eq:taucovering}
N\left(\cF,\left\|\cdot\right\|_{L_2\left(P\right)},\tau \left\|F\right\|_{L_2\left(P\right)}\right) \le {\left(\frac{A}{\tau}\right)}^{v}
\end{align}
\noindent where $N\left(\cT,\hat d,\epsilon \right)$ is defined as the minimal number of open $\hat d$-balls of radius $\epsilon$ and centers in $\cT$ required to cover $\cT$. For any $t_1$ and $t_2$, $\left\|\frac{K\left(\frac{t_1-\cdot}{h}\right)}{g\left(\cdot\right)}-
\frac{K\left(\frac{t_2-\cdot}{h}\right)}{g\left(\cdot\right)}\right\|_{L_2\left(P\right)} \le \frac{1}{g_{min}}\left\|K\left(\frac{t_1-\cdot}{h}\right)-
K\left(\frac{t_2-\cdot}{h}\right)\right\|_{L_2\left(P\right)}$. Let  $B_{\cF}\left(t_0,\delta \right) \triangleq \{t:\left\|K\left(\frac{t-\cdot}{h}\right)-K\left(\frac{t_0-\cdot}{h}\right)\right\|_{L_2\left(P\right)} \le \delta \}$, and $B_{\cF_g}\left(t_0,\delta \right) \triangleq \{t:\left\|\frac{K\left(\frac{t-\cdot}{h}\right)}{g\left(\cdot\right)}-\frac{K\left(\frac{t_0-\cdot}{h}\right)}{g\left(\cdot\right)}\right\|_{L_2\left(P\right)} \le \delta \}$. Then $B_{\cF}\left(t_0,\delta \right) \subseteq B_{\cF_g}\left(t_0,\frac{\delta}{g_{min}} \right)$.

We choose the envelope function for $\cF_g$ as $F_g=\frac{F}{g_{min}}$ and $\left|u_g\right| \le F_g$ for any $u_g \in \cF_g$, then
\begin{align}\label{eq:taucovering1}
N\left(\cF_g,\left\|\cdot\right\|_{L_2\left(P\right)},\tau \left\|F_g\right\|_{L_2\left(P\right)}\right) \le {\left(\frac{A}{\tau}\right)}^{v}
\end{align}
\noindent So that $\cF_g$ is also a bounded VC class. By similar arguement in the proof of Theorem~\ref{theorem::kdeconvergence}, we can obtain (\ref{eq:gkdeconvergence}).
\qed
\end{proof}

\begin{MyTheorem}\label{theorem::pluginerrortheorem1}
Under the assumption (A1), (B), (C1)-(C2), the error of the plug-in classifier $F_n^{PI}$ satisfies
\begin{align}\label{eq:pluginerrorboundG}
&R\left(F_n^{PI}\right) \le \frac{2}{n^2}\sum\limits_{l,m} {{\theta _{lm}}{G_{lm,h_n}}}+{\cO}\left({h_n^{\gamma}}\right)
\end{align}
\noindent where ${{\theta _{lm}}} = {\1}_{\{Y_l \ne Y_m\}}$ is a class indicator function and
\begin{align}\label{eq:Glm}
&{G_{lm,h_n}} = G_{h_n}\left(X_l,X_m\right), \; G_{h}\left(x,y\right) = \frac{{{K_{{h}}}\left( {{x} - {y}} \right)}}{{\hat f_{n,h}^{\frac{1}{2}}\left(x\right)} {\hat f_{n,h}^{\frac{1}{2}}\left(y\right)}}
\end{align}
\noindent for any kernel bandwidth sequence
$\{h_n\}_{n=1}^{\infty}$ that satisfies assumption (B).
\end{MyTheorem}
\begin{proof}\label{eq:bayeserrortheorem1seg1}
For any $1 \le i,j \le Q$, by (\ref{eq:etaerror})
\begin{align}\label{eq:etahat2eta}
{{{\E}_X}\left[ {{\hat \eta_n^{(i)}}\left( X \right){\hat \eta_n^{(j)}}\left( X \right)} \right]} \le  {{{\E}_X}\left[ {{\eta^{(i)}}\left( X \right){\eta^{(j)}}\left( X \right)} \right]} + \cO\left(h_n^{\gamma}\right)
\end{align}
Since
$${\E}_{X}\left[\eta^{(i)}\left(X\right)\eta^{(j)}\left(X\right)\right]=\int_{\cX} {\frac{\pi^{(i)}f^{(i)}\left(x\right)}{f^{\frac{1}{2}}\left(x\right)}\cdot
\frac{\pi^{(j)}f^{(j)}\left(x\right)}{f^{\frac{1}{2}}\left(x\right)}} dx,$$

$f$, $f^{(i)}$, and $f^{(j)}$ are H\"{o}lder-$\gamma$ smooth, and $f^{\frac{1}{2}}$ is also H\"{o}lder-$\gamma$ smooth. For any sequence $\{\tilde h_n\}_{n=1}^{\infty}$ that satisfies assumption (B), we obtain the kernel estimator $\tilde \eta_n^{(i)}$ of $\frac{\pi^{(i)}f^{(i)}\left(x\right)}{f^{\frac{1}{2}}\left(x\right)}$ using the generalized kernel density estimator (\ref{eq:gkde}):
\begin{align}\label{eq:bayeserrortheorem1seg2}
&\tilde \eta_n^{(i)}\left(x\right) = \frac{1}{n}\sum\limits_{l=1}^n \frac{K_{\tilde h_n}\left(x-X_l\right){\1}_{\{Y_l = i\}}}{f^{\frac{1}{2}}\left(X_l\right)}
\end{align}
By Lemma~\ref{lemma::gkdelemma}, $\left\|\tilde \eta^{(i)}-\frac{\pi^{(i)}f^{(i)}\left(x\right)}{f^{\frac{1}{2}}\left(x\right)}\right\|_{\infty} = \cO\left(h_n^{\gamma}\right)$ for $1 \le i \le Q$.\footnote{It can be verifed by applying Lemma~\ref{lemma::gkdelemma} to $\{X_l,Y_l\}_{l=1}^n$.} It follows that
\begin{align}\label{eq:bayeserrortheorem1seg3}
&{\E}_{X}\left[\eta^{(i)}\left(X\right)\eta^{(j)}\left(X\right)\right] \le {{\E}_{X}\left[\tilde \eta_n^{(i)}\left(X\right){\tilde \eta}_n^{(j)}\left(X\right)\right]}+\cO\left(h_n^{\gamma}\right)
\end{align}
Also,
\begin{align*}
&\sum\limits_{i,j = 1,...,Q,i \ne j}{{\E}_{X}\left[\tilde \eta_n^{(i)}\left(X\right){\tilde \eta}_n^{(j)}\left(X\right)\right]}\\ \nonumber
&=\frac{1}{n^2}\sum\limits_{l,m}{\frac{\int_{\cX}{K_{\tilde h_n}\left(x-X_l\right)K_{\tilde h_n}\left(x-X_m\right)}}{f^{\frac{1}{2}}\left(X_l\right)f^{\frac{1}{2}}\left(X_m\right)} \sum\limits_{i,j = 1,...,Q,i \ne j}{{\1}_{\{Y_l=i\}}{\1}_{\{Y_m=j\}}}} \\ \nonumber
&=\frac{1}{n^2}\sum\limits_{l,m}{\frac{K_{{\sqrt 2}\tilde h_n}\left(X_l-X_m\right)}{f^{\frac{1}{2}}\left(X_l\right)f^{\frac{1}{2}}\left(X_m\right)} \theta_{lm}}
\end{align*}
\noindent and we obtain last equality by convolution of two Gaussian kernels. Let $h_n = {\sqrt 2}\tilde h_n$, then $h_n$ satisfies assumption (B). The kernel density estimator ${{\hat f}_{n,h_n}}$ satisfies $\left\| {{{\hat f}_{n,h_n}} - f} \right\|_{\infty} = {\cO}\left({h_n^{\gamma}}\right)$ by Theorem~\ref{theorem::kdeconvergence}. Therefore, with $n$ large enough,
\begin{align}\label{eq:f2fhat}
&\left|{\sum\limits_{i,j = 1,...,Q,i \ne j}{{\E}_{X}\left[\tilde \eta_n^{(i)}\left(X\right){\tilde \eta}_n^{(j)}\left(X\right)\right]}-\frac{1}{n^2}\sum\limits_{l,m}{G_{lm,h_n} \theta_{lm}}}\right| \\ \nonumber
&= \frac{1}{n^2}\left|\sum\limits_{l,m}{\frac{K_{ h_n}\left(X_l-X_m\right)}{f^{\frac{1}{2}}\left(X_l\right)f^{\frac{1}{2}}\left(X_m\right)} \theta_{lm}}-\sum\limits_{l,m}{\frac{K_{h_n}\left(X_l-X_m\right)}{\hat f_{n,h_n}^{\frac{1}{2}}\left(X_l\right)\hat f_{n,h_n}^{\frac{1}{2}}\left(X_m\right)} \theta_{lm}}\right| \\ \nonumber
&\le \frac{1}{n^2}\sum\limits_{l,m}{K_{ h_n}\left(X_l-X_m\right)\frac{\left|f^{\frac{1}{2}}\left(X_l\right)f^{\frac{1}{2}}\left(X_m\right)-\hat f_{n,h_n}^{\frac{1}{2}}\left(X_l\right)\hat f_{n,h_n}^{\frac{1}{2}}\left(X_m\right)\right|}{\hat f_{n,h_n}^{\frac{1}{2}}\left(X_l\right)\hat f_{n,h_n}^{\frac{1}{2}}\left(X_m\right)f^{\frac{1}{2}}\left(X_l\right)f^{\frac{1}{2}}\left(X_m\right)}} \\ \nonumber
&\le \frac{1}{n^2}\sum\limits_{l,m}{K_{ h_n}\left(X_l-X_m\right)\frac{f_{max}^{\frac{1}{2}}{\mathcal O}\left({h_n^{\gamma}}\right)}{2f_{min}^{\frac{9}{2}}}} \\ \nonumber
&=\frac{f_{max}^{\frac{1}{2}}{\mathcal O}\left({h_n^{\gamma}}\right)}{2f_{min}^{\frac{9}{2}}}\cdot \frac{1}{n} \sum\limits_{l=1}^n {\hat f_{n,h_n}\left(X_l\right)} = {{\cO}\left({h_n^{\gamma}}\right)} \nonumber
\end{align}
\noindent It follows from (\ref{eq:etahat2eta}), (\ref{eq:bayeserrortheorem1seg3}) and (\ref{eq:f2fhat}) that
\begin{align*}
&R_n^{PI} = 2\sum\limits_{i,j = 1,...,Q,i \ne j}{{{\E}_X}\left[ {{\hat \eta_n^{(i)}}\left( X \right){\hat \eta_n^{(j)}}\left( X \right)} \right]}
=
\frac{2}{n^2}\sum\limits_{l,m}{{G_{lm,h_n}}{\theta_{lm}}}+{\cO}\left({h_n^{\gamma}}\right)
\end{align*}
\noindent and (\ref{eq:pluginerrorboundG}) is verified.
\qed
\end{proof}
\begin{MyRemark}\label{remark::pluginerrorremark}
If we further assume that $f^{\alpha}$ is H\"{o}lder-$\gamma$ smooth for some $0 \le \alpha \le 1$, then $G_h\left(x,y\right)$ in (\ref{eq:Glm}) can be
\begin{align}\label{eq:Glmalpha}
&{G_h\left(x,y\right)} = \frac{{{K_{{h}}}\left( {{x} - {y}} \right)}}{{\hat f_{n,h}^{\alpha}\left(x\right)} {\hat f_{n,h}^{1-\alpha}\left(y\right)}}
\end{align}
\end{MyRemark}

Define $V\left(F_n^{PI}\right)$ as the volume of the region in $\cX$ misclassified by $F_n^{PI}$. Using almost the same argument in Theorem~\ref{theorem::pluginerrortheorem1}, we have the tight upper bound for $V\left(F_n^{PI}\right)$:
\begin{MyTheorem}\label{theorem::pluginerrortheorem2}
Under the assumption (A1), (B), (C1)-(C2), the volume of the region in $\cX$ misclassified by $F_n^{PI}$ satisfies
\begin{align}\label{eq:pluginerrorboundV}
&V\left(F_n^{PI}\right) \le \frac{2}{n^2}\sum\limits_{l,m} {{\theta _{lm}}{V_{lm,h_n}}}+{\cO}\left({h_n^{\gamma}}\right)
\end{align}

\noindent where
\begin{align}\label{eq:Vlm}
&{V_{lm,h_n}} = V_{h_n}\left(X_l,X_m\right),\; V_h\left(x,y\right)=\frac{{{K_{{h}}}\left( {{x} - {y}} \right)}}{{\hat f_{n,h}\left(x\right)} {\hat f_{n,h}\left(y\right)}}
\end{align}
\noindent for any bandwidth sequence
$\{h_n\}_{n=1}^{\infty}$ that satisfies assumption (B).
\end{MyTheorem}

\subsection{Connection to Low Density Separation}
Low Density Separation \cite{Chapelle05}, a well-known criteria for clustering, requires that the cluster boundary should pass through regions of low density. Suppose the data $\{X_i\}_{i=1}^n$ lies on a domain $\Omega  \subseteq {R^d}$. Let $f$ be the probability density function on $\Omega$,  $S$ be the cluster boundary which separates $\Omega$ into two parts $S_1$ and $S_2$. Following the Low Density Separation assumption, \cite{NarayananBN06} suggests that the cluster boundary $S$ with low weighted volume $\int\limits_S {f\left( s \right)} ds$ should be preferable. \cite{NarayananBN06} also proves that a particular type of cut function converges to the weighted volume of $S$. By slight change of their proof, we obtain the following result relating the error of the plug-in classifier and the weighted volume of the cluster boundary.
\begin{MyLemma}\label{lemma::ldlemma}
For any kernel bandwidth sequence $\{h_n\}_{n=1}^{\infty}$ such that $\mathop {\lim }\limits_{n \to \infty } {h_n} = 0$ and $h_n > n^{-\frac{1}{4d+4}}$,  with probability $1$,
\begin{align}\label{eq:ldcut}
\mathop {\lim }\limits_{n \to \infty } {\frac{1}{n^2}\frac{\sqrt {2\pi} }{h_n}}{\sum\limits_{l,m} {{\theta _{lm}}{G_{lm,h_n}}}} = \int\limits_S {f\left( s \right)} ds
\end{align}
\end{MyLemma}
Combining Lemma~\ref{lemma::ldlemma} and Theorem~\ref{theorem::pluginerrortheorem1},~\ref{theorem::pluginerrortheorem2}, we have
\begin{MyTheorem}\label{theorem::ldcuttheorem}
Under the assumption (A1) and (C1)-(C2), for any kernel bandwidth sequence $\{h_n\}_{n=1}^{\infty}$ satisfying assumption (B) and $h_n > n^{-\frac{1}{4d+4}}$,
\begin{align}\label{eq:ldcut1}
&\mathop {\lim }\limits_{n \to \infty } \frac{R\left(F_n^{PI}\right)}{h_n\int\limits_S {p\left( s \right)} ds} \le \frac{\sqrt{2\pi}}{\pi}
\end{align}
Also, the constant on the RHS of (\ref{eq:ldcut1}) is the best (cannot be smaller).
\end{MyTheorem}
\begin{proof}
The conclusion follows from the fact that $\mathop {\lim }\limits_{n \to \infty } {R\left(F_n^{PI}\right)} \le \mathop {\lim }\limits_{n \to \infty }\frac{2}{n^2}\sum\limits_{l,m} {{\theta _{lm}}{G_{lm}}}$ and the bound is tight.
\qed
\end{proof}

Theorem~\ref{theorem::ldcuttheorem} shows that the error of the plug-in classifier is bounded from above by the weighted volume of the cluster boundary (scaled by $h_n$). It is also worth noting that $R\left(F_n^{PI}\right)={o}\left({\int\limits_S {f\left( s \right)} ds}\right)$.

\subsection{Connection to Diffusion Maps}
Consider the complete graph $\cG$ whose vertices are associated with the data lying on a submanifold of $\R^d$, and the weight of the edge between data $x$ and $y$ is denoted by $e_h\left(x,y\right)$. $e_h\left(x,y\right)$ is determined by the similarity kernel induced by the bound for the misclassification error or the volume of the misclassified region of the unsupervised plug-in classifier, namely $G_h\left(x,y\right)$ or $V_h\left(x,y\right)$ defined in (\ref{eq:Glm}) and (\ref{eq:Vlm}) respectively. Empirically methods such as unsupervised SVM suggest clustering be performed by minimizing the error of unsupervised classification $R\left(F_{S_C}\right)$. Minimizing the error bound for the plug-in classifier (\ref{eq:pluginerrorboundG}) (or (\ref{eq:pluginerrorboundV})) is equivalent to finding the (normalized) minimum-cut in the graph $\cal G$, which is solved by computing the normalized graph Laplacian. Let $d_h\left(x\right)=\int\limits_{\cX}{e_h\left(x,y\right)f\left(y\right)}dy$, the normalized graph Laplacian computes the anisotropic kernel $p_h\left(x,y\right)=\frac{e_h\left(x,y\right)}{d_h\left(x\right)}$.
It is shown in \cite{Coifman06} that the forward infinitesimal operator of the corresponding Markov chain is
\begin{align}\label{eq:infinitesimaloperator}
{\cH}^{\left(\alpha\right)}\phi=\Delta \alpha-\frac{\Delta\left(f^{1-\alpha}\right)}{f^{1-\alpha}}\phi
\end{align}
\noindent where $\Delta$ is the Laplace-Beltrami operator.

If $e_h\left(x,y\right)=G_h\left(x,y\right)$, then $\alpha=\frac{1}{2}$ in (\ref{eq:infinitesimaloperator}), the forward infinitesimal operator reduces to
\begin{align}\label{eq:FokkerPlankOperator}
{\cH}^{\left(\frac{1}{2}\right)}\phi=\Delta \phi-\frac{\Delta\left(\sqrt f\right)}{\sqrt f}\phi
\end{align}
\noindent which yields the backward Fokker-Planck operator.

Moreover, if $e_h\left(x,y\right)=V_h\left(x,y\right)$, then $\alpha=1$ in  (\ref{eq:infinitesimaloperator}), and the forward infinitesimal operator is
\begin{align}\label{eq:LaplaceBeltramiOperator}
{\cH}^{\left(1\right)}\phi=\Delta \phi
\end{align}
\noindent which is the Laplace-Beltrami operator. In this case, the corresponding Markov chain converges to the Brownian motion, and the normalized graph Laplacian disregards the data distribution and only captures the Riemannian geometry of the data. It is consistent with the choice of the similarity kernel $V\left(x,y\right)$ which reflects only the geometric property (volume of the misclassified region) of the data.

\section{Conclusion}\label{sec::conclusion}
Empirical unsupervised classification methods gain satisfactory practical results, but few methods consider the error of unsupervised classifiers. We study the misclassification error of unsupervised classification by two popular classifiers, i.e. the nearest neighbor classifier (NN) and the plug-in classifier, and build the connection between the error of unsupervised plug-in classifier and the weighted volume of cluster boundary. The normalized graph Laplacian from the similarity kernel induced by the unsupervised plug-in classifier recovers the Fokker-Planck operator and the Laplace-Beltrami operator, revealing close relationship to different types of Diffusion maps.

\bibliographystyle{splncs}
\bibliography{mybib}

\end{document}